\documentclass{article}

\usepackage{microtype}
\usepackage{graphicx}
\usepackage{subfig}
\usepackage{booktabs} 
\usepackage{caption}
\usepackage{subcaption}

\usepackage{hyperref}

 \usepackage[accepted]{icml2024}

\usepackage{amsmath}
\usepackage{amssymb}
\usepackage{mathtools}
\usepackage{amsthm}

\usepackage[capitalize,noabbrev]{cleveref}

\theoremstyle{plain}
\newtheorem{theorem}{Theorem}[section]
\newtheorem{proposition}[theorem]{Proposition}
\newtheorem{lemma}[theorem]{Lemma}
\newtheorem{corollary}[theorem]{Corollary}
\theoremstyle{definition}

\theoremstyle{remark}
\newtheorem{remark}[theorem]{Remark}

\usepackage[textsize=tiny]{todonotes}

\icmltitlerunning{Deformation Stability of Equivariant Convolutional Representations through Multi-layered Kernel Representations}

\begin{document}

\twocolumn[
\icmltitle{Stability Analysis of Equivariant Convolutional Representations Through The Lens of Equivariant Multi-layered CKNs}



\icmlsetsymbol{equal}{*}

\begin{icmlauthorlist}
\icmlauthor{Soutrik Roy Chowdhury}{yyy}
\end{icmlauthorlist}

\icmlaffiliation{yyy}{Arenberg Doctoral School, KU Leuven, 3001 Leuven, Belgium}

\icmlcorrespondingauthor{Soutrik Roy Chowdhury}{roychowdhurysoutrik@gmail.com}

\begin{center}
\textbf{Editors:} \textit{S. Vadgama, E. J. Bekkers, A. Pouplin, O. Kaba, H. Lawrence, R. Walters, T. Emerson, H. Kvinge, J. M. Tomczak, S. Jegelka}
\end{center}

\icmlkeywords{Machine Learning, ICML}

\vskip 0.3in
]



\printAffiliationsAndNotice{}  

\begin{abstract}
In this paper we construct and theoretically analyse group equivariant convolutional kernel networks (CKNs) which are useful in understanding the geometry of (equivariant) CNNs through the lens of reproducing kernel Hilbert spaces (RKHSs). We then proceed to study the stability analysis of such equiv-CKNs under the action of diffeomorphism and draw a connection with equiv-CNNs, where the goal is to analyse the geometry of inductive biases of equiv-CNNs through the lens of reproducing kernel Hilbert spaces (RKHSs). Traditional deep learning architectures, including CNNs, trained with sophisticated optimization algorithms is vulnerable to additive perturbations, including `adversarial examples'. Understanding the RKHS norm of such models through CKNs is useful in designing the appropriate architecture and can be useful in designing robust equivariant representation learning models.
\end{abstract}

\section{Introduction}

In the past decade deep neural networks, especially convolutional neural networks (CNNs) \cite{lecuncnn89} have achieved impressive results for various predictive tasks, notably in the domains of computer vision \cite{krizhevsky2017imagenet} and natural language processing. Much success of CNNs in these domains relies on (1) the availability of large scaled labeled and structured data which allow the model to learn huge number of parameters without worrying too much of overfitting, and (2) the ability to model local information of signals (e.g., images) at multiple scales, while also representing the signals with some invariance through pooling operations. The latter property of CNNs have distinguished them from fully-connected networks \cite{li2021why} in terms of sample efficiency, generalization ability and computational speed, much through its elegant model design. Still, understanding the exact mathematical nature of this invariance as well as the characteristics of the functional spaces where CNNs live are indeed open problems for which multiple constructions and analyses have been provided in past years.

One such construction is of group equivariant CNNs \cite{cohen2016group} where the translation equivariance of convolutional layers has been generalized to other kinds of symmetries, for e.g., rotations, reflections, etc., thus making CNNs equivariant to more general transformations, where such transformations and corresponding equivariant maps for learning layerwise features are encoded by the representation theory of finite symmetric groups, an important tool used by mathematicians and physicists for centuries. Despite different elegant constructions of group equivariant CNNs, for e.g. \cite{cohen2016steerable, weiler20183d,weiler2019general} there exists only a few works, e.g., \cite{cohen2019general,kondor2018generalization} focusing on the theoretical analysis of such networks, which might be beneficial to understand the geometry of these inductive biases in the model that plays pivotal role in the enhanced expressive power of the equivariant convolutional networks.

Another construction is of Convolutional Kernel Networks (CKNs) \cite{mairal2014convolutional,mairal2016end} where local signal neighbourhoods are mapped to points in a reproducing kernel hilbert space (RKHS) through the kernel trick and then hierarchical representations are built by composing kernels with corresponding RKHSs (patch extraction + kernel mapping + pooling operations in each layer) which is equivalent to construction of a sequence of feature maps in conventional CNNs, but of infinite dimension. A wider functional space approach \cite{bietti2019group} of CKNs has been proposed for multi-dimensional signals which also admits multilayered and convolutional kernel structure. This functional space also contains a large class of CNNs with homogeneous activation functions, thus showing such CNNs can also enjoy same theoretical properties that of CKNs, therefore highlighting on the geometry of the functional spaces in which CNNs lie. Furthermore, an analysis of approximation and generalization capabilities of deep convolutional networks through the lens of CKNs has been performed in \cite{bietti2022approximation}. Despite such mathematical analysis, exploring the equivariance properties of CKNs as well as generalization capabilities and robustness of equiv-CKNs have not been performed in details.

In this paper we first study how to make convolutional kernel layers equivariant to actions by a locally compact group $G$. Following the notations of diffeomorphism stability \cite{mallat2012group} we analyse the stability bounds of equiv-CKNs which depends upon the equivariant architecture of CKNs and corresponding RKHSs norms, thus providing a notion of robustness of equiv-CKNs. We then give an intuition on the (geometric) complexity of equivariant CNNs (equiv-CNNs) by giving a rough outline on how to construct equiv-CNNs in RKHSs, that might be helpful in studying stability and generalization properties of equiv-CNNs by bounding their corresponding RKHS norm.

\textbf{Contributions.}
\begin{itemize}
    \item We construct group equivariant multi-layered CKNs in details and provide a general analysis of how to make a CKN equivariant to any compact group action through \cref{thm: equiv-ckns}, followed by examples of such equiv-CKNs.
    \item Following the definition of deformation stability from \cite{mallat2012group}, we provide a Lipschitz stability styled bound of equivariant convolutional kernel representations in \cref{lie-transformation} thus showing how much robust equiv-CKNs are to the action of local diffeomorphism.
    \item We outlined how to extend the construction of group equiv-CKNs to group equiv-CNNs which is useful to extend the studies performed on equiv-CKNs (e.g., robustness, generalization bounds) to equiv-CNNs.
\end{itemize}

\subsection{Related works} 

The main source of motivation of this work on equivariant CKNs and corresponding stability bound is \cite{bietti2019group}, where authors generalized the construction of CKNs \cite{mairal2016end} and provide stability analysis and (equi)-invariance properties of CKNs. Though the authors provided a group invariant construction of CKNs, a detailed construction analysis with examples as well as stability properties of such generalized equiv-CKNs are still missing which is done in this work. Authors as well as us took the approach of \cite{mallat2012group} to study stability of deep convolutional kernel representations with respect to diffeomorphic actions. The motivation of the analysis is based upon the results from classical harmonic analysis. The approach of \cite{mallat2012group} uses pre-defined filters whereas ours is an end to end equivariant filters learning approach. 

The idea of learning equivariant functions with kernels was first conceived in \cite{JMLR:v8:reisert07a}, where the authors learned equivariant filters with matrix valued kernels. Recently \cite{lang2021a} classified the group steerable kernels for group CNNs through Wigner-Eckart theorems. The approach in this paper is different as our construction relies on properties of RKHSs and traditional kernel methods \cite{Smolakernel}. 

We note that deformation robustness of roto-translation equiv-CNNs has been studied in \cite{gao2022deformation}. The approach is different from ours as it relies on the idea of decomposed convolutional filters \cite{qiu2018dcfnet}. Moreover we studied deformation stability of any group equivariant CKNs, going beyond the domain of $\mathbb{R}^2 \rtimes SO(2)$, as done in that paper. Furthermore, a recent work \cite{schuchardt2023provable} have studied the effect of learning with adversarial examples on equivariant neural networks. We understand that our approach is different from the one proposed. Nevertheless we will take these approaches into account while studying our generalized equivariant convolutional kernel networks, going beyond Euclidean domain to manifolds and graphs.  

\section{Group Equivariant Convolutional Kernel Networks}
\label{grp equiv-CKNs}

The construction of a multilayered CKN involves transforming an input signal $x_0 \in L^2(\mathbb{R}^d, \mathcal{H}_0)$ (for e.g., $\mathcal{H}_0 = \mathbb{R}^{p_0}$, where for a 2D RGB image $p_0=3$ and $d=1$ and $x_0(u)$ in $\mathbb{R}^2$ represents the RGB pixel value at location $u \in \mathbb{R}^2$) into a sequence of feature maps, $x_k$'s in $L^2(\mathbb{R}^d, \mathcal{H}_k)$, by building a sequence of RKHSs $\mathcal{H}_k$'s, for each $k$, where a new feature map $x_k$ is built from the previous one $x_{k-1}$ by consecutive application of patch extraction $P_k$, kernel mapping $M_k$ and linear pooling $A_k$ operators, as shown in \cref{ckn-schematic}. For a detailed construction of multilayered CKNs on continuous and discrete signal\footnote{Note that though here in our construction signals are considered continuous for a better theoretical analysis, however for practical purposes one needs to discretize the feature maps.} domains we refer readers to \cite{bietti2019group,mairal2016end}.
\begin{figure}[ht]
\vskip 0.2in
\begin{center}
\centerline{\includegraphics[width=\columnwidth]{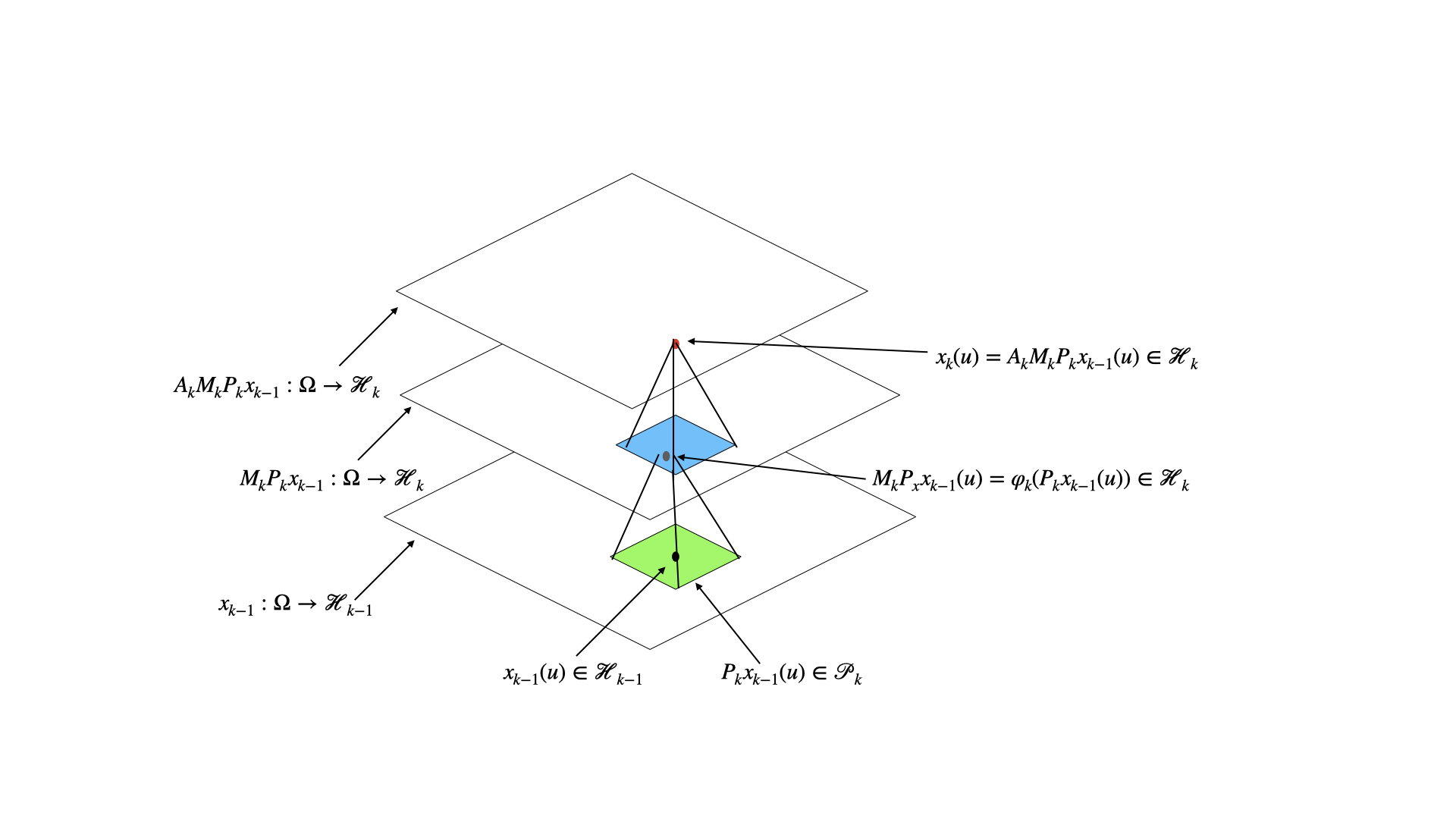}}
\caption{A schematic diagram of $1$-layer of a CKN where one constructs $k$-th signal representation from the $k$$-1$-th one in a RKHS $\mathcal{H}_k$ through patch extraction, kernel mapping and pooling operators, as similarly shown in \cite{bietti2019group}. Signal domain $\Omega = \mathbb{R}^d$ (in this figure $d=2$) on which locally compact group $G$ acts. One can construct a multilayered CKN by stacking these layers in a hierarchical manner and make the entire network equivariant by making each layers equivariant to the action of $G$.}
\label{ckn-schematic}
\end{center}
\vskip -0.2in
\end{figure}

In (section 3.1, \cite{bietti2019group}) it is shown that CKNs are equivariant to the translations as the layers commute with the action of translations, much like its classical CNNs counterpart. Following the general notations of group equivariance in CNNs \cite{kondor2018generalization} through the notion of locally compact group actions, it is possible to encode other kind of equivariance to group transformations (e.g., rotations, reflections) in CKN layers by constructing equivariant $P_k$'s, $M_k$'s and $A_k$'s for each $k$ that commutes with the action of a group of transformation $G$. We assume $G$ is locally compact so that we can define a Haar measure $\mu$ on it.\footnote{$\mu$ satisfies $\mu(gS) = \mu(S)$ for any Borel set $S \subseteq G$ and $g \in G$. Considering a Haar measure on $G$, which always exists for locally compact groups, the integration at pooling layers become invariant to group actions, as discussed briefly in appendix A.3 in \cite{cohen2019general}.} The action of an element $g \in G$ is denoted by operator $L_g$ where $L_g x(u) = x(g^{-1}u)$. We also assume that every element $x(u) \in \mathbb{R}^d$ can be reached with a transformation $u_\omega \in G$ from a neutral element, say $\hat{x}_0(u) \in \mathbb{R}^d$. One can then extend the original signal $\hat{x}$ by defining $x(u) = \hat{x}(u_\omega \cdot \hat{x}_0(u))$, as similarly shown in \cite{kondor2018generalization,bietti2019group}. Then one has
\begin{equation}
\label{transformation}
\begin{split}
    L_gx(u_\omega) & = x(g^{-1}u_\omega) \\
   & = \hat{x}((g^{-1}u_\omega) \cdot \hat{x}_0(u)) = \hat{x}(g^{-1} \cdot x(u)),
\end{split}
\end{equation}
where $\cdot$ denotes the group action and hence transformed signals preserve the structure of $\hat{x}$. With the input signals now defined on the locally compact group $G$, one can define layerwise equivariant patch extraction, kernel mapping and pooling operators at each layer $k$ which are outlined below.

\textbf{Patch extraction operator.}  Patch extraction operator $P_k: L^2(G, \mathcal{H}_{k-1}) \to L^2(G, \mathcal{P}_k)$ is defined for all $u \in G$ as 
\begin{equation}
\label{patch-equation}
    P_kx_{k-1}(u) \coloneqq (x_{k-1}(uv))_{v \in S_k},
\end{equation}
where $S_k \subseteq G$ is a patch shape centered at the identity element of $G$ and $\mathcal{P}_k \coloneqq L^2(S_k,\mathcal{H}_{k-1})$ is a Hilbert space equipped with the norm $||x||^2 = \int_{S_k} ||x(u)||^2 d\mu_k(u)$, where $d\mu_k$ is the normalized Haar measure on $S_k$'s. $P_k$ commutes with $L_g$ as one can show
\begin{equation*}
\begin{split}
    P_kL_gx_{k-1}(u)  & = (L_gx_{k-1}(uv))_{v \in S_k} = (x(g^{-1}uv))_{v \in S_k} \\ 
             & = P_kx_{k-1}(g^{-1}u) = L_gP_kx_{k-1}(u)
\end{split}
\end{equation*}

\textbf{Kernel mapping operator.} Kernel operator $M_k : L^2(G,\mathcal{P}_k) \to L^2(G,\mathcal{H}_k)$, for all $u \in G$, is defined as
\begin{equation}
\label{kernel-equation}
    M_kP_kx_{k-1}(u) \coloneqq \varphi_k(P_kx_{k-1}(u)),
\end{equation}
where $\varphi_k : \mathcal{P}_k \to \mathcal{H}_k$ is the kernel mapping associated to a positive definite kernel $K_k$ operating on the patches. Like \cite{mairal2016end}, we define the dot product kernel $K_k$ as
\begin{equation*}
    K_k(x,x') = ||x|| ||x'|| k_k \left(\frac{\langle x,x' \rangle}{||x||||x'||} \right),  x,x' \neq 0,
\end{equation*}
which is positive definite because a Maclaurin expansion with only non-negative coefficients \cite{Smolakernel} can be constructed from $k_k$. A choice of dot product kernels are listed in \cite{bietti2019group}. As $M_k$ is a pointwise operator, thus it commutes with $L_g$.

We define a function $k_k: [-1,+1] \to \mathbb{R}$ such that $k_k(u) = \sum_{i=0}^\infty b_iu^i$ such that $b_i \geq 0$ for all $i$ and $k_k(1) =1$ and $0 \leq k_k'(1) \leq 1$, where $k_k'$ is the first order derivative of $k_k$. Then we define the kernel $K_k$ on $\mathcal{P}_k$ as
\begin{equation}
\label{positive-definite kernel}
    K_k(x,x') \coloneqq ||x||||x'||k_k\left(\frac{\langle x,x' \rangle}{||x||||x'||} \right),
\end{equation}
when $x,x' \in \mathcal{P}_k \backslash \{0\}$, and $K_k(x,x')=0$ if either of $x$ and $x'$ is $0$. Note that $K_k$ is positive definite as $k_k$ admits a Maulaurin series with only non-negative coefficients \cite{Smolakernel}. Then the kernel mapping $\varphi_k(\cdot)$, associated to the positive definite kernel $K_k$ is denoted by $K_k(x,x') = \langle \varphi_k(x),\varphi_k(x') \rangle$.

\textbf{Norm preservation of operator $M_k$.} The constraint $k_k(1)=1$ ensures that $M_k$ preserves the norm, as, $||\varphi_k(x)|| = K_k(x,x)^{1/2} = ||x||$ leads us to $||M_kP_kx_{k-1}|| = ||P_kx_{k-1}||$ for any $k$, and therefore $M_kP_kx_{k-1} \in L^2(G,\mathcal{H}_k)$.

\textbf{Non-expansiveness of $\varphi_k(\cdot)$'s.} In order to study the stability results we need our kernel mapping non-expansive, i.e., $||\varphi_k(x) - \varphi_k(x')|| \leq ||x-x'||$\footnote{It is however possible to extend the non-expansiveness of kernel mapping to any Lipschitz continuous functions.}, for $x,x' \in \mathcal{P}_k$, and the constraint on the derivative of $k_k$'s, i.e., $0 \leq k_k'(1) \leq 1$ ensures that it is always going to hold. The following lemma states the non-expansivess of the kernel mapping.

\begin{lemma}[Lemma 1, \cite{bietti2019group}]
   Let $K_k$ be a positive-definite kernel given by \cref{positive-definite kernel} which satisfies the constraints given by $k_k$'s. Then the RKHS mapping $\varphi_k : \mathcal{P}_k \to \mathcal{H}_k$, for all $x,x' \in \mathcal{P}_k$ satisfies $||\varphi_k(x) - \varphi_k(x')|| \leq ||x-x'||$. Moreover $K_k(x,x') \geq \langle x,x' \rangle$, i.e., the kernel $K_k$'s are lower bounded by the linear kernels.
\end{lemma}

\textbf{Pooling operator.} Pooling operator $A_k : L^2(G,\mathcal{H}_k) \to L^2(G,\mathcal{H}_k)$, for all $u \in G$, is defined as
\begin{equation}
\label{pooling-int}
\begin{split}
     x_k(u) = A_kx_k(u) & \coloneqq \int_G x_k(uv)h_k(v)d\mu(v) \\ 
               & = \int_G x_k(v)h_k(u^{-1}v)d\mu(v),
\end{split}
\end{equation}
where $h_k$ is the pooling filter at layer $k$\footnote{Note that \cref{pooling-int} is a type of Bochner integral when $\mathcal{H}$ is infinite dimensional.} following similar construction from \cite{raj2017local}. One typical example of such pooling filter is Gaussian pooling filter which is given by $h_{\sigma_k}(u) \coloneqq \sigma_k^{-d}h_k(u/\sigma_k)$, where $\sigma_k$ is the scale of the pooling filter and $h_k(u) = (2\pi)^{-d/2}exp(-|u|^2/2)$. Following it's definition it is easy to show that $A_k$ commutes with $L_g$, i.e., one can show that $A_kL_gx_k(u) = L_gA_kx_k(u)$ for all $g \in G$ and therefore at each layer, all operators are equivariant to the action of $G$.

Note that the definitions of equivariant operators at each layers follow the similar construction of G-convolution with respect to a locally compact group (section 4, \cite{kondor2018generalization}). Here the subgroups $H_k$ are the patches $S_k$, which are Borel sets, according to our assumptions. Our representation $x_k(u)$ for each layer $k$ can be stacked into a full representation of a $N$-layer CKN as $x_N(u) = \Phi_N(x) \coloneqq A_NM_NP_NA_{N-1}M_{N-1}P_{N-1}\cdot\cdot\cdot A_1M_1P_1x_0(u)$. Our construction already shows that each layer of a $N$-layers CKN is $G$-equivariant, and establishing the equivariance of an entire CKN $\Phi_N$, i.e., $L_g\Phi_Nx(u) = \Phi_NL_gx(u)$ is a straightforward job as $\Phi_N$ is formed by stacking $G$-equivariant pooling, kernel and patching layers. Adding a non-linear activation map $\sigma$ in the end still makes a $N$-layered predictive CKN model equivariant. 

\textbf{Generalized Equivariant convolutional kernel representations.} Note the term `convolution' in equiv-CKNs comes from the definition of pooling filter which resembles with the definition of classical convolutional mapping and in line with the generalized convolutional operator defined on compact groups by \cite{kondor2018generalization} which is given by $(f \ast_g h) = \int_G f(uv^{-1})h(v)d\mu(v)$, where $f$ and $h$ are functions defined on $G$ and the integration is with respect to the Haar measure $\mu$. Note how our pooling filter is in a convolution with the feature map $x_k(\cdot)$'s. The following theorem shows that it is also possible to construct an group equivariant CKN from a standard CKN in a RKHS by choosing appropriate homogeneous patches and carefully designing the pooling layer, which we believe is more generalized approach to construct an equivariant CKN and also will be helpful in understanding the construction of equivariant convolutional networks in RKHSs.

\begin{theorem}[Equivariance of a CKN]
\label{thm: equiv-ckns}
    Let $G$ be a locally compact group and $\Phi_N$ be a (N+1)-layered CKN\footnote{The first layer is the input layer $x_o(u)$, where one can use downsampling with a factor $\sigma_0$ for high frequency data. And the final layer is the final pooling layer $A_N$.}, following the standard construction of a CKN \cite{mairal2016end}. Let the patches $S_k \subseteq G$ form the index sets $\chi_k = G/S_k$, which are homogeneous spaces of $G$, given by group action operators $L_g$, on which patch extraction operator $P_k$ is evaluated. The pooling operators $A_k$ are in generalized convolution with the non-linear feature maps $x_k$ for each $k \in 1,..,N$, i.e., $A_k(x_k) = x_k \ast_g h_k$, where $h_k$ is the pooling filter associated to $A_k$'s and the point-wise non-linearity is appearing from the kernel mapping $\varphi_k$, if and only if the CKN $\Phi_N$ is equivariant with respect to locally compact group $G$'s action on it's inputs.
\end{theorem}

\begin{proof}
    Suppose we translate $x_{k-1}(u)$ with some $g \in G$ and obtain $\hat{x}_{k-1}(u)$ where $\hat{x}_{k-1}(u) = x_{k-1}(g^{-1} \cdot u)$. We apply patch operator $P_k$ on  $\hat{x}_{k-1}(u)$ with patches collected from $\chi_k$. Applying $M_k$ we get $\varphi_k\hat{x}_{k-1}(u)=\hat{x}_k(u)$.
    Then, $A_k\hat{x}_k(u) = \hat{x}_k \ast_g h_k = \int_G \hat{x}_k(uv^{-1})h_k(v)d\mu(v) = \int_G x_k(g^{-1}uv^{-1})h_k(v)d\mu(v)) = (x_k \ast_g h_k)(g^{-1}u) = A_kx_k(g^{-1}u)$. Then by induction we can gradually show equivariance of the entire CKN $\Phi_N$.

    For the reverse direction we closely follow the arguments from \cite{kondor2018generalization} which draws significant amount of representation theoretic analysis of generalized convolution. We ask readers to check \cref{Grp-equiv-ckns-constructions} for the detailed proof.
\end{proof}

Note that, one can always express an equivariant convolutional kernel map in an convolution-like integral (theorem 3.1, \cite{cohen2019general}) which also supports our construction of group equiv-CKNs on homogeneous space. A direct consequence of \cref{thm: equiv-ckns} is the following.

\begin{corollary}[Equivariant convolutional kernels in RKHS]
    \cref{pooling-int} can always be written as cross-correlation between the feature map and the pooling filter. Moreover in equiv-CKNs, representation, $\Phi_N(x) \in L^2(G,\mathcal{H}_N)$ is equivariant (with respect to $G$) if and only if each $\varphi_k$'s are in cross-correlation with an equivariant pooling filter.
\end{corollary}

We note that a classification of equivariant kernels in CNNs are done in \cite{lang2021a}, such as understanding spherical harmonics, which can be used to represent an infinite dimensional representations on Hilbert space. This idea, especially used in constructing equivariant kernels for SO(3), SE(3) can be used in construction of equivariant convolutional kernel networks, however our kernels are here dot product kernels with non-expansiveness assumptions. This is a basic difference with the ideas of equivariant kernels used in group equiv-CNNs and group equiv-CKNs. 

A general theory of equiv-CNNs on homogeneous space is given through the notions of vector bundles, fiber space, and fields in \cite{cohen2019general} where equivariant maps between feature spaces are shown to be in one-to-one correspondence with equivariant convolutions, obtained by the space of equivariant kernels (convolution is all you need). As one can define vector bundles and fibers on Hilbert space \cite{bertram1998reproducing,takesaki2003theory} we believe that similar notions of equivariant convolution maps can also be deducted for equiv-CKNs, though the latter already contains notion of equivariant kernels through the definitions of $M_k$'s and $A_k$'s. We will work on these in our follow-up studies.

\subsection{Examples of group equivariant CKNs.}
\label{examples-ckns}

Below we provide some examples of equivariant CKNs under different compact group actions.

\textbf{SO(3)-equivariant CKNs.} The group elements in 3D rotation group $G=SO(3)$ are $R_{\theta}$ where $R_{\theta}$ is a rotation matrix in $SO(3)$. We define group action on an element $u \in \mathbb{R}^3$ as $g\cdot u = R_{\theta}u$, for some angle $\theta$, whereas $g^{-1} \cdot u = -R_{-\theta}u$. By considering a normalized Haar measure on unit $S^2$, one can use \cref{transformation} to transform a signal $x(u) \in L^2(\mathbb{R}^3)$ in $L^2(SO(3))$ while preserving the signal information. 

We define a patch shape $S_k$ consisting of $\{R_{\theta}\}$'s centered around $\mathbb{I} \in SO(3)$, on which one can define patch extraction operator $P_k$. There is no restrictions on $M_k$ as it is a pointwise operator and we just need a suitable dot product kernel for that. The pooling layers $A_k: L^2(G) \to L^2(G)$ are defined as $A_kx(g) = \int_G x(g\cdot R_{\theta})h_k(R_{\theta})dR_{\theta}$, where $h$ is the Gaussian pooling filter with a bandwidth $\sigma_k$ defined on $\mathbb{R}^3$.

\textbf{SE(3)-equivariant CKNs.} 3D Roto-translation group SE(3) can be viewed as a semi-direct product between $\mathbb{R}^3$ and SO(3), i.e., $G= SE(3) = \mathbb{R}^3 \rtimes SO(3)$. Group operation on $SE(3)$ is defined as $gg' = (v+R_{\theta}v', R_{\theta+\theta'})$, where $R_{\theta}$ is the rotation matrix in $SO(3)$, for $g=(v,R_{\theta})$ and $g'=(v',R_{\theta'})$. The action of a group element $g=(v,R_{\theta})$ on a signal $ u \in \mathbb{R}^3$ is defined as $g\cdot u = v + R_{\theta+\theta'}u$, for some $\theta' \in [0,2\pi)$, whereas $g^{-1} \cdot u = -R_{\theta+\theta'}(v-u)$. Using the same argument as in previous case one can extend a signal $x(u) \in L^2(\mathbb{R}^3)$ to $L^2(G)$, where the left invariant Haar measure is defined as $d\mu(v,R_{\theta}) = dvd\mu_c(R_{\theta})$. $dv$ is Lebesgue measure on $\mathbb{R}^3$ and $d\mu_c(R_{\theta})$ is normalized Haar measure on unit $S^2$.

A patch shape $S_k$ can be defined as $S_k = \{(v,\mathbb{I})\}$, where $v \in \mathbb{R}^3$ and $\mathbb{I}$ is the identity element of group $SE(3)$, on which one can define patch operators. Pooling operator is defined as $A_kx(g) = \int_G x(g(v,\mathbb{I}))h_k(v)dv$, where $h$ is the Gaussian pooling filter with a bandwidth $\sigma_k$ defined on $\mathbb{R}^3$.

\textbf{Spherical CKNs.} Here $G = SO(3), H=SO(2)$, whereas the homogeneous space is the quotient space $S^2 = SO(3)/SO(2)$. Extending a signal from $L^2(\mathbb{R}^3)$ to $L^2(S^2)$ requires one to define an invariant Haar measure $d\mu(R_{\theta})d\mu(\theta')$, where $d\mu(R_{\theta})$ is the normalized Haar measure on unit $S^2$ and $\mu(\theta')$ is the normalized Haar measure on unit circle $S^1$.

Our patches can be defined as rotation matrix elements from $SO(2)$, centered around the identity element of subgroup $SO(2)$. The pooling operator on $L^2(S^2)$ is defined as $A_kx(r,\theta,\eta) = \int_{SO(3)}x((r,\theta,\eta)\cdot(\theta',\eta'))h_k(\theta',\eta')d\mu(R_{\theta})d\mu(\eta')$, where $(r,\theta,\eta)$ is an element in $S^2$.

\section{Stability Analysis of Equivariant CKNs}
\label{Stability equiv-CKNs}

Following our construction of equiv-CKNs in the previous section we now proceed to understand the stability of the equivariant kernel representations under the action of diffeomorphisms, which might be beneficial to get robustness of equiv-CKNs against adversarial examples \cite{bietti2019kernel}. Moreover stability against small deformation is desirable for most deep learning models and serves as a basic receipe in building geometric deep learning models, as stated in \cite{bronstein2017geometric}. We follow the notion of deformation and stability from \cite{mallat2012group} which is defined as a $C^1$-diffeomorphism $\tau: \mathbb{R}^d \to \mathbb{R}^d$ through a linear operator $L_\tau$ as $L_\tau x(u) = x(u-\tau(u))$ and we say that a representation $\Phi(\cdot)$ is stable under the actions of $\tau$ if there exist non-negative constants $C_1$ and $C_2$ such that
\begin{equation}
\label{stability_gen}
    || \Phi(L_\tau x) - \Phi(x)|| \leq (C_1 ||\nabla\tau||_\infty + C_2||\tau||_\infty)||x||,
\end{equation}
where $\nabla\tau$ is the Jacobian of $\tau$ and $||\cdot||$ is the $L^2$-operator norm and $||\nabla\tau||_\infty \coloneqq sup_{u \in \mathbb{R}^d} ||\nabla \tau(u)))$ and $||\tau||_\infty \coloneqq sup_{u \in \mathbb{R}^d} |\tau(u)|$, where $|\cdot|$ is the standard Euclidean norm on $\mathbb{R}^d$. We also assume $||\nabla\tau||_\infty \leq 1/2$ in order to keep the deformation invertible and avoid degenerate situations, as assumed in \cite{mallat2012group}. 

We are interested in the stability of convolutional kernel representations $\Phi_N$. For a semi-direct product group $G \coloneqq \mathbb{R}^d \rtimes H$ \cite{weiler2019general} we state the stability bound of kernel representations for $G \coloneqq \mathbb{R}^d \rtimes H$, where each $g \in G$ is given by $g=(u,\hat{h})$, where $u \in \mathbb{R}^d$ and $\hat{h} \in H$ and the group action $L_g$ on the signals are given by $L_g x(u) = x(g^{-1}\cdot u) = x((g^{-1}\cdot u, \hat{h}(\hat{h}')^{-1})) = x(g^{-1}(u,\hat{h}))$, where $\hat{h}'$ in an element of subgroup $H$.

\begin{lemma}
\label{lemma:Stability-of-ckns}
    If $||\nabla\tau||_\infty \leq 1/2$ and $\text{sup}_{c \in \hat{S}_k} |c| \leq \kappa\sigma_{k-1}$, where patch shape $S_k = \{(u,0)\}_{u \in \hat{S}_k} \subseteq G$ with $\hat{S}_k \subset \mathbb{R}^d$, $\sigma_{k-1}$, the scale of pooling filter $h_{k-1}$ at layer $k-1$, and $\kappa$ is the patch size, and $0$ is the identity element of the subgroup $H \subseteq G$. Then we have
    \begin{equation}
    \label{norm-one}
     ||[P_kA_{k-1}, L_\tau]|| \leq C_1||\nabla\tau||_\infty,
    \end{equation}
    where $C_1$ depends upon $h_{k-1}$ and $\kappa$ and $L_\tau x((u,\hat{h})) = x((\tau(u),0)^{-1}(u,\hat{h}))$. Similarly we have 
    \begin{equation}
    \label{norm-two}
     ||L_\tau A_N - A_N|| \leq \frac{C_2}{\sigma_N} ||\tau||_\infty, 
    \end{equation}
    where $C_2 = 2^2 \cdot ||\nabla h_N||$ and $\nabla h_N$ is the gradient of the last pooling filter $h_N$.
\end{lemma}

\begin{proof}
    Note that for all $k$ we have
    \begin{equation*}
    \begin{split}
        P_kx_{k-1}((u,\hat{h})) & = x((uv,\hat{h} \cdot 0))_{v \in \hat{h}\hat{S}_k} \\
                       & = x((uv,\hat{h}))_{v \in \hat{h}\hat{S}_k},
    \end{split}    
    \end{equation*}
    where $\hat{h}\hat{S}_k$ is in $S_k \subseteq G/H$, and by $\hat{h} \cdot 0$, we meant the group composition with the identity element.

    Similarly we have $A_kx_k((u,\hat{h})) = \int_G x_k((v,\hat{h}'))h_k((u,\hat{h})^{-1}v)d\mu(v) = \int_{\mathbb{R}^d} x_k((v,\hat{h}))h_k(u^{-1}v)d\mu(v)$ which follows from the second term of \cref{pooling-int}. Moreover as $G/H \simeq \mathbb{R}^d$, we can integrate over $G/H \simeq \mathbb{R}^d$ by using integral over $G$, i.e., $\int_{\mathbb{R}^d} f(x) dx = \int_G f(gH)dg$.

    For a fixed $\hat{h} \in H$ we can obtain signal $\hat{x} \coloneqq x(\cdot,\hat{h}) \in L^2(\mathbb{R}^d,\mathcal{H}_0)$ from the signal $x \in L^2(G,\mathcal{H}_0)$, and we have corresponding operators $\Tilde{P}_k$, $\Tilde{A}_k$ and $\Tilde{L_\tau}$ now defined on $L^2(\mathbb{R}^d)$, with a transformed patch $\Tilde{S}_k = \hat{h}\hat{S}_k$ for $\Tilde{P}_k$.

    Then for $x \in L^2(G,\mathcal{H}_0)$, we have,
    \begin{multline*}
    ||[P_kA_{k-1}, L_\tau]x||_{L^2(G)}^2 \\
         = \int_G ||([P_kA_{k-1},L_\tau]x)(\cdot,\hat{h})||_{L^2(\mathbb{R}^d)}^d d\mu(\hat{h}) \\
         = \int_{\mathbb{R}^d} ||[\Tilde{P}_k\Tilde{A}_{k-1},\Tilde{L}_\tau](\hat{x})||_{L^2(\mathbb{R}^d)}^2 d\mu(\hat{h}) \\
        \leq \int_{\mathbb{R}^d} ||[\Tilde{P}_k\Tilde{A}_{k-1},\Tilde{L}_\tau]||^2 ||(\hat{x})||_{L^2(\mathbb{R}^d)}^2d\mu(\hat{h}) \\
         \leq \left( sup ||[\Tilde{P}_k\Tilde{A}_{k-1},\Tilde{L}_\tau]||^2 \right) ||x||_{L^2(G)}^2,
     \end{multline*}
    so that one has $||[P_kA_{k-1},L_\tau]||_{L^2(G)} \leq sup ||[\Tilde{P}_k\Tilde{A}_{k-1},\Tilde{L}_\tau]||_{L^2(\mathbb{R}^d)}$. As we have assumed that $sup_{c \in \hat{S}_k} |c| \leq \kappa\sigma_{k-1}$, so we can bound each of $||[\Tilde{P}_k\Tilde{A}_{k-1},\Tilde{L}_\tau]||$ as shown in section 3.1 of \cite{bietti2019group}\footnote{Interested readers can read appendix C.4. for proof of the lemma and detailed understanding of deformation stability of classical CKNs.} for detailed understanding of deformation stability of classical CKNs by bounding the operator norms when signals are in $L^2(\mathbb{R}^d)$ which is possible as one can bound $||[\Tilde{P}_k\Tilde{A}_{k-1},\Tilde{L}_\tau]||$ with $sup_{c \in \hat{S}_k}||[L_c\Tilde{A}_{k-1},\Tilde{L}_\tau]||$ and showing $[L_c\Tilde{A}_{k-1},\Tilde{L}_\tau]$ is an integral operator, one can bound its norm via Schur's test. \cref{norm-one} is then obtained by applying the bound derived for classical CKNs.

    Similarly by applying lemma 2.11 from \cite{mallat2012group} one obtains upper bound on $||L_\tau A_N - A_N||_{L^2(G)}$ by first restricting it on $||\Tilde{L}_\tau \Tilde{A}_N - \Tilde{A}_N||_{L^2(\mathbb{R}^d)}$ and then applying the lemma 2.11 get the desired result, given by \cref{norm-two}.
\end{proof}

Here while studying the bounds on operator norm, \cref{norm-one} of \cref{lemma:Stability-of-ckns} is stated on the norm of the commutators of operators, given by $[A,B] = AB-BA$. It shows that commutators are stable to diffeomorphism $\tau$, as the norm is controlled by $||\nabla\tau||_\infty$, whereas the second norm in \cref{norm-two} decays with the last pooling bandwidth $\sigma_N$. Note that for the semi-direct group $G$ we restrict the diffeomorphism on the field, $\mathbb{R}^2$ with the assumption that the elements of subgroup $H$ remains unaffected by the deformation $\tau$ or has negligible effect.

\begin{theorem}[Stability bound]
\label{thm:Stability-bound}
Subsequently we have
\[    ||\Phi_N(L_\tau x) - \Phi_N(x)|| \leq \]
\begin{equation}
  \left( C_1(1+N)||\nabla\tau||_\infty + \frac{C_2}{\sigma_N}||\tau||_\infty\right) ||x||. 
\end{equation}
\end{theorem}

The bound is immediately followed by combining \cref{bounds-operator-norm}\footnote{Check \cref{Stability-proofs} for proposition 8.1.} with \cref{norm-one} and \cref{norm-two} which are extracted by bounding the corresponding operator norms.

From \cref{thm:Stability-bound} and \cref{lemma:Stability-of-ckns} we understand that stability to deformation of a CKN representation depends linearly on the depth of network, the patch size (smaller the better) and pooling filter whereas $C_2$ controls the global invariance of network under deformation and is inversely proportional to last layer's pooling filter bandwidth, $\sigma_N$. One needs to have small $C_2$ in order to have global equivariant representation and indeed it's small as $\sigma_N$ typically increases exponentially with the number of layers $N$. We note that it is possible to extend the stability analysis to any $G$ and $H$. There are standard ways defining diffeomorphism on compact Lie groups, and thus the diffeomorphism operator can be generally defined on $G$ as $\tau_{lie}: G \to G, L_{\tau_{lie}}x(u) = x(u-\tau^{-1}_{lie}\cdot u)$. By adding a global pooling layer at the end, defined as $A: L^2(G) \to L^2(\mathbb{R}^d), Ax(u) = \int_G x(g^{-1}\cdot u)d\mu_{lie}(u)$, where $d\mu_{lie}$ is an appropriate Haar measure on the respective Lie group $G$, we can additionally obtain equivariance of the CKN $\Phi_N$ with respect to Lie group transformation along with the stability bounds. 

\begin{proposition}
\label{lie-transformation}
    With $||\nabla \tau_{lie}||_\infty \leq 1/2$ and $\text{sup}_{c \in S_k} |c| \leq \kappa\sigma_{k-1}$, where $\kappa$, $S_k$'s and $\sigma_{k-1}$ follows the same definition from \cref{lemma:Stability-of-ckns}, for any $g \in G$ we have,
\begin{multline}
    ||L_gA\Phi_N(L_{\tau_{lie}}x) - A\Phi_N(L_gx)|| \\
    \leq ||\Phi_N(L_{\tau_{lie}}x)-\Phi_N(L_gx)|| \\
    \leq \left( C_1(1+N)||\nabla\tau_{lie}||_\infty + \frac{C_2}{\sigma_N}||\tau_{lie}||_\infty\right) ||x||.
\end{multline}
\end{proposition}

\begin{proof}
   One can write $A$ as an integral operator as  
   \[ Ax(v) = \int_G x(g^{-1}\cdot v)k(u,v)d\mu_{lie}(v)  \]
   where $K(u,v) = \delta_u(v) =1$, $\delta_u$ is a Dirac delta operator. Then $\int |k(u,v)|d\mu_{lie}(v) = \int |k(u,v)|d\mu_{lie}(u) =1$ implies $||A|| \leq 1$, followed by Schur's test. As $L_g$ is a continous operator between two normed spaces, it is bounded and hence $||L_gA|| \leq ||A|| \leq 1$.

   From the construction, $L_g$ and $L_{\tau_{lie}}$ commute and hence using the fact that $\Phi_N$ is equivariant to the action of $G$ we get the first part of the inequality, whereas the second part of the inequality follows from \cref{lemma:Stability-of-ckns} and \cref{thm:Stability-bound}. 
\end{proof}

We note that similar results are stated in \cite{bietti2019group} for 2D roto-translation groups where global rotation invariance is attained through stating a global pooling layer. \cref{lie-transformation} also shows how much equivariant operator is affected by diffeomorphism operator which would establish the measure of equivariance \cite{gruver2022lie} of equivariant networks under adversarial training. Generalization of equiv-CNNs beyond known symmetries have been studied in \cite{finzi2020generalizing} and we hope further detailed analysis would complement the construction of convolutional representations equivariant with respect to any Lie group transformation, discussed in that work.

\subsection{Some empirical studies with the stability analysis of equiv-CKNs.}

In this section we do some empirical analysis on the stability bounds of equiv-CKNs stated above, with aims to 1) understand the role of bandwidth of pooling filters, patch size $\kappa$, choice of kernels, scale of deformation, and 2) compare the results with classical translation only equivariant CKNs on some benchmark equivariant datasets.

\textbf{Experimental setups.} We select $SE(2)=\mathbb{R}^2 \rtimes SO(2)$ and $SO(3)$ as our groups for construction of group equiv-CKNs. For $SE(2)$ and $SO(3)$ we respectively pick rotated MNIST described in \cite{weiler2018learning} and rotated MNIST on $S^2$ with stereographic projection described in \cite{cohen2018spherical} as our datasets. For simplicity in full kernel computation we select $N=2$.

In order to implement on grid space we need discretization of our equiv-CKNs and training with manifold optimization. For the latter, viz., training\footnote{The objective is similar to the structural risk minimization.} of equivariant CKNs, we use adaptive stochastic gradient descent on manifold \cite{AbsMahSep2008} by projecting kernel representations on $S^2$. We ask the readers to read from Mairal's work \cite{mairal2016end} which we simply follow for computation of our baseline CKNs. Some useful information are also made available in \cref{Grp-equiv-ckns-constructions}.

We parametrize the deformation map $\tau$ with a scale $\alpha$, as done in \cite{bietti2019group}, defined as $L_{\alpha\tau}x(u) = x(u-\alpha\tau(u)) \approx x(u)-\alpha\tau(u)\nabla x(u)$. Here $\alpha$ controls the amount of deformation. We pick a reference image from the dataset and then using 5 different values of $\alpha$, deform it into another 5 images. From rotated MNIST we pick 4 randomly picked reference images from each image class and then using 5 different $\alpha$'s to transform into 5 deformed images. Together we have 40 reference images and 200 generated deformed images. We then compute the `mean relative distance' in the representation space between a reference image and i) all 20 generated deformed images from the same class, ii) 50 generated deformed images combining different classes randomly picked from the class of 200 images. We then average our result for all 40 reference images.

Given a model $M$ and a set of images $S$, mean relative distance between an image $x$ and $S$ is given by,
\begin{equation}
    \frac{1}{|S|}\sum_{x' \in S} \frac{||\Phi_M(x') - \Phi_M(x)||}{||\Phi_M(x)||}
\end{equation}

\begin{figure*}
\vskip 0.2in
\begin{center}
\includegraphics[width=.3\textwidth]{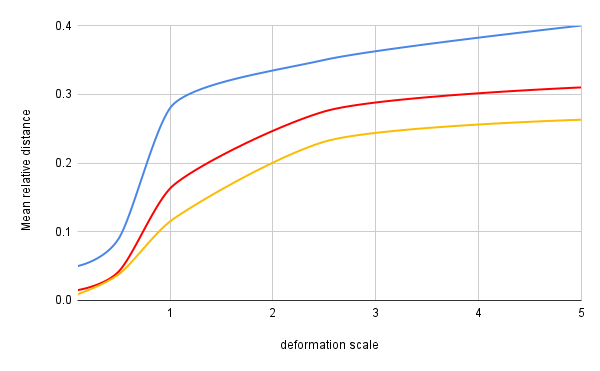}\hfill
\includegraphics[width=.3\textwidth]{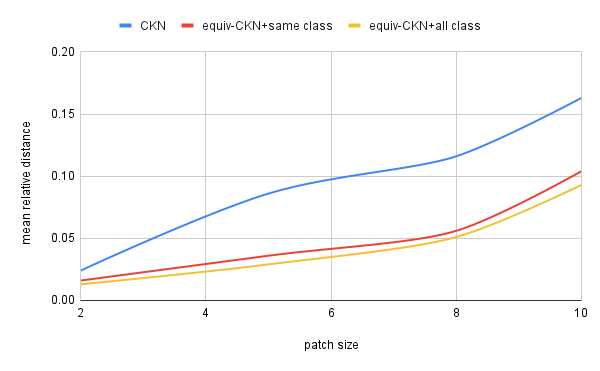}\hfill
\includegraphics[width=.3\textwidth]{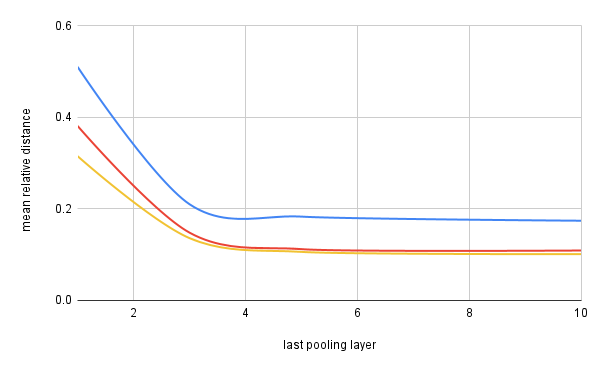}
\\[\smallskipamount]
\includegraphics[width=.30\textwidth]{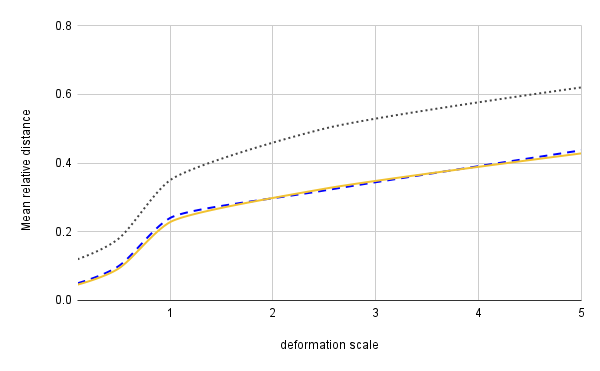}\hfill
\includegraphics[width=.30\textwidth]{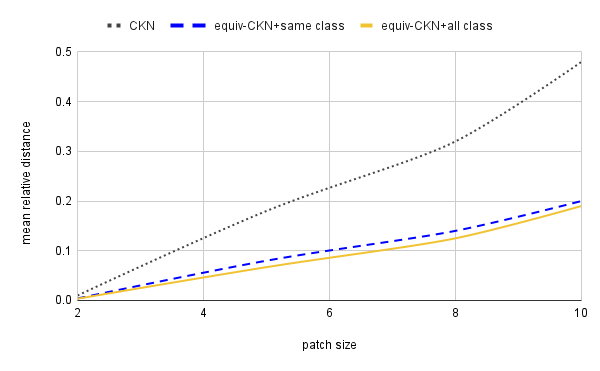}\hfill
\includegraphics[width=.30\textwidth]{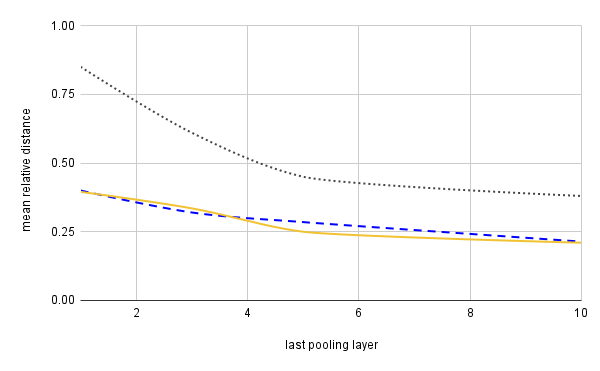}
\caption{Stability analysis with equiv-CKNs and comparing with CKNs. The first row represent experiments with rotated MNIST with $G=SE(2)$, whereas the second column is experiments on rotated MNIST on sphere $S^2$ with $G=SO(3)$. We evaluate mean average distances while varying deformation scale $\alpha =\{ 0.1,0.5,1,2.5,5\}$, patch size $\kappa= \{ 2,5,8,10\}$ and scale of last pooling layer $h_k$, $\sigma_k=\{ 1,3,5,10\}$. For experiments with patch size and last pooling layer parameter $\sigma$, we keep $\alpha=1$ and choose RBF kernel with bandwidth $\{ 5,10\}$ first column, and exponential kernel, $k_{exp}(\langle x,x'\rangle) = exp(\langle x,x'\rangle -1)$ for our kernel mapping for the second column.}
\label{ckn-results}
\end{center}
\vskip -0.2in
\end{figure*}

In \cref{ckn-results} we note that group equiv-CKNs outperform the classical CKNs in deformation stability analysis in terms of the computed `mean relative distance'. For equiv-CKNs trained with same label (i) and with all labels (ii), the performance of training with all labels are slightly better in case of rotated MNIST ($G=SE(2)$) and almost same performance in case of spherical CKNs. Classical CKNs performance got worse on spherical MNIST.

Regarding the choice of kernels, we note that performance largly depends upon how efficiently we can compute the full kernel representations. RBF, exponential, arc-cosine with degree 1, and polynomial kernel with degree 2,3 have relatively same performance and have not much effect on stability analysis. Computational time of full kernel matrix grows rapidly $O(N^2)$ with the increase in number of layers $N$. There are methods (for e.g., \cite{rahimi2007random}) to efficiently compute large scale kernel matrices, however discussion on efficient computation of equiv-CKNs is currently out of scope for this paper.

\section{Equivariant Convolutional Networks in RKHSs}
\label{gen_equiv-CKNs}

In this section we give an outline on how to construct an equivariant G-CNN $f$ \cite{cohen2016group} recursively from intermediate functions $\hat{f}_k^i$ that lie in the RKHSs $\mathcal{H}_k$ which is of the form,
\begin{equation}
\label{equiv-cnns-rkhs}
    \hat{f}_k^i (x) = ||x||\sigma(\langle w_k^i, x\rangle /||x||),
\end{equation}
primarily used to study embedding of CNNs\footnote{CNNs with homogeneous activation function $\sigma$'s are considered. For e.g., smoothed-ReLU function.} in RKHSs and thus extending theoretical results of CKNs to CNNs. Here $w_k^i$'s are convolutional filters used to obtain intermediate feature maps $\hat{f}_k^i$'s followed by non-linear activation maps ($\sigma$'s) and linear pooling, similarly as defined in \cref{grp equiv-CKNs}. We would like to point out how one can embed an equiv-CNNs in RKHS and thus enjoying the analysis of CKNs.

\subsection{Construction of group equiv-CNN $f$ in the RKHS} 
\label{construction-formalities}

One defines the $k$-th layer of equiv-CNN function $f$ in $\mathcal{H}_k$ from the $(k-1)$-th layer as follows: For an input signal $x_0 \in L^2(G,\mathcal{H}_0 \coloneqq \mathbb{R}^{p_0})$, we build a sequence of feature maps, $x_k \in L^2(G,\mathcal{H}_k \coloneqq \mathbb{R}^{p_k})$ with $p_k$ channels. We use the following intermediate functions $g_k^i \in \mathcal{P}_k$ and $f_k^i \in \mathcal{H}_k$, where $i=1,...,p_k$ and construct it from the $(k-1)$-th intermediate function inductively, where the intermediate functions are of form \cref{intermediate}.
\begin{align*}
    g_k^i(u) & = \sum_{h \in S_k}\sum_{j=1}^{p_{k-1}} w_k^{ij}(u^{-1}h)f_{k-1}^j (x(h)) \\
    f_k^i(x(u)) & = ||x(u)|| \sigma\left( \langle g_k^i, x(u) \rangle / ||x(u)|| \right),
\end{align*}
for $x(u) \in \mathcal{P}_k \backslash \{0\}$, $u \in G$, and the filters $w_k^i(u) = (w_k^{ij}(u))_{j=1,...p_{k-1}}$ are equivariant through the definition of the intermediates and also matches the notion of group equivariant correlation of \cite{cohen2016group}.

With this construction one can show that the equivariant feature maps $x_k$ are given are $x_k^i(u) = \langle f_k^i, M_kP_kx_{k-1}(u) \rangle$, where $u \in G$ and $P_k$ and $M_k$'s are our patch and kernel operators, respectively, used to define an equiv-CKN. With a final linear prediction layer one can immediately show that an equivariant CNN lies in a RKHS, supported by \cref{linear-function}. We will work on the detailed construction in our follow-up paper, discussing in depth the generalization bounds and sample complexity of equiv-CKNs.

Following proposition 13 and proposition 14 of \cite{bietti2019group} one get upper bounds on the RKHS norm of classical CNNs $f_\sigma$ which is given by the parameter of the final linear fully connected layer, the spectral norm of the convolutional filter parameters at each layers and the choice of the activation function. One can think of similar bounds for equiv-CNNs through it's RKHSs norm given by the final pooling layer (or the norm of the global pooling operator $A_c: L^2(G) \to L^2(\mathbb{R})$ defined for $x \in L^2(G)$ as $A_c x(u) = \int_G x(g^{-1}u)d\mu_c(g)$), equivariant filters and the choice of non-linearities. One can use spectral norms to study generalization, for e.g., done in \cite{bartlett2017spectrally}, of equiv-CNNs. We give an intuitive analysis of generalization bounds through Rademacher complexity of the equiv-CNNs (CKNs) function classes, in \cref{generalization-Rademacher}. Similarly as one can do stability analysis of CNNs through the Lipschitz smoothness and given by the relation through Cauchy-Schwarz's inequality, 
\begin{equation}
\label{lipschitz-smoothness}
    |f_\sigma(L_\tau x) - f_\sigma(x)| \leq ||f_\sigma||_{\mathcal{H}_N}||\Phi_N(L_\tau x) - \Phi_N (x)||_{L^2(G)},
\end{equation}
where $||\cdot||_\mathcal{H}$ is the standard Hilbert norm, one can then extend the same for equiv-CNNs, outlined in \cref{equiv-cnns-rkhs}, and supported by \cref{thm:Stability-bound}. In \cite{cisse2017parseval} it is shown that robustness to adversarial examples of deep models can be achieved by bounding the Lipschitz smoothness. Invariant and Equivariant CKNs have already possessed the Lipschitz stability property and hence the above equation can be useful to construct adversarially robust equivariant convolutional representations.

\section{Conclusion and Future Work Directions}

We have shown how to construct a hierarchical kernel network for multilayered equivariant representation learning by constructing the equivariant feature maps in RKHSs. Then we studied the stability bounds of equiv-CKNs under some mild assumptions and through the Lipschitz stability which shows the stability with respect to a deformation depends upon the specific architecture of equiv-CKNs including the depth of the network and most importantly of the RKHS norm, which acts as an implicit regularizer in our model and controlling the norm leads to better stable model, as shown in \cite{bietti2019kernel}. Finally we outlined the possibility of embedding a group equiv-CNN into a RKHS and thus extending the studies of equivariant convolutional networks through the lens of equiv-CKNs that might provide novel insights on equivariant convolutions as well as on deep multilayered equivariant kernel networks, for e.g., shown in context of classical CNNs in \cite{anselmi2015deep}.

Despite we follow the common framework of \cite{kondor2018generalization} and expect such equiv-CKNs can also be defined on spherical domain \cite{cohen2018spherical} it may not be possible to define the same framework on a general manifold. One needs careful construction of gauge equivariant CKNs, following similar works on gauge equiv-CNNs \cite{cohen2019gauge,de2020gauge} which might be possible as anisotropic kernel (e.g., indefinite kernels, asymmetric kernels) representations can be modelled through reproducing kernel Banach space (RKBS) or Krein space(RKKS) etc., to name of few. This is a future work we are interested to work on. 

We are also interested to do a thorough analysis of generalization capability of equivariant networks under adversarial training through analysing the generalization bounds of equiv-CKNs. A PAC-Bayesian generalization analysis has been performed recently on equivariant networks \cite{behboodi2022pac}, whereas \cite{bietti2022approximation} has studied generalization of 2-layers CKNs by bounding the excessive risk for the kernel ridge regression (KRR) estimator. Analyzing the generalization bounds of the equiv-CKNs with these approaches is indeed a promising direction of research.

\subsection*{Acknowledgement}

The author would like to thank anonymous reviewers for passing valuable feedback and comments on the paper which helps in improving the structure of the paper and clarifications as well as shedding light on potential future directions. Part of the work was supported by Johan Suykens's ERC advanced grant E-DUALITY (grant number 787960). 

\nocite{langley00}

\bibliography{example_paper}

\begin{thebibliography}{40}
\providecommand{\natexlab}[1]{#1}
\providecommand{\url}[1]{\texttt{#1}}
\expandafter\ifx\csname urlstyle\endcsname\relax
  \providecommand{\doi}[1]{doi: #1}\else
  \providecommand{\doi}{doi: \begingroup \urlstyle{rm}\Url}\fi

\bibitem[Absil et~al.(2008)Absil, Mahony, and Sepulchre]{AbsMahSep2008}
Absil, P.-A., Mahony, R., and Sepulchre, R.
\newblock \emph{Optimization Algorithms on Matrix Manifolds}.
\newblock Princeton University Press, Princeton, NJ, 2008.
\newblock ISBN 978-0-691-13298-3.

\bibitem[Anselmi et~al.(2015)Anselmi, Rosasco, Tan, and Poggio]{anselmi2015deep}
Anselmi, F., Rosasco, L., Tan, C., and Poggio, T.
\newblock Deep convolutional networks are hierarchical kernel machines.
\newblock \emph{arXiv preprint arXiv:1508.01084}, 2015.

\bibitem[Bartlett et~al.(2017)Bartlett, Foster, and Telgarsky]{bartlett2017spectrally}
Bartlett, P.~L., Foster, D.~J., and Telgarsky, M.~J.
\newblock Spectrally-normalized margin bounds for neural networks.
\newblock \emph{Advances in neural information processing systems}, 30, 2017.

\bibitem[Behboodi et~al.(2022)Behboodi, Cesa, and Cohen]{behboodi2022pac}
Behboodi, A., Cesa, G., and Cohen, T.~S.
\newblock A pac-bayesian generalization bound for equivariant networks.
\newblock \emph{Advances in Neural Information Processing Systems}, 35:\penalty0 5654--5668, 2022.

\bibitem[Bertram \& Hilgert(1998)Bertram and Hilgert]{bertram1998reproducing}
Bertram, W. and Hilgert, J.
\newblock Reproducing kernels on vector bundles.
\newblock \emph{Lie Theory and Its Applications in Physics III}, pp.\  43--58, 1998.

\bibitem[Bietti(2022)]{bietti2022approximation}
Bietti, A.
\newblock Approximation and learning with deep convolutional models: a kernel perspective.
\newblock In \emph{International Conference on Learning Representations}, 2022.
\newblock URL \url{https://openreview.net/forum?id=lrocYB-0ST2}.

\bibitem[Bietti \& Mairal(2019)Bietti and Mairal]{bietti2019group}
Bietti, A. and Mairal, J.
\newblock Group invariance, stability to deformations, and complexity of deep convolutional representations.
\newblock \emph{The Journal of Machine Learning Research}, 20\penalty0 (1):\penalty0 876--924, 2019.

\bibitem[Bietti et~al.(2019)Bietti, Mialon, Chen, and Mairal]{bietti2019kernel}
Bietti, A., Mialon, G., Chen, D., and Mairal, J.
\newblock A kernel perspective for regularizing deep neural networks.
\newblock In \emph{International Conference on Machine Learning}, pp.\  664--674. PMLR, 2019.

\bibitem[Bronstein et~al.(2017)Bronstein, Bruna, LeCun, Szlam, and Vandergheynst]{bronstein2017geometric}
Bronstein, M.~M., Bruna, J., LeCun, Y., Szlam, A., and Vandergheynst, P.
\newblock Geometric deep learning: going beyond euclidean data.
\newblock \emph{IEEE Signal Processing Magazine}, 34\penalty0 (4):\penalty0 18--42, 2017.

\bibitem[Cisse et~al.(2017)Cisse, Bojanowski, Grave, Dauphin, and Usunier]{cisse2017parseval}
Cisse, M., Bojanowski, P., Grave, E., Dauphin, Y., and Usunier, N.
\newblock Parseval networks: Improving robustness to adversarial examples.
\newblock In \emph{International conference on machine learning}, pp.\  854--863. PMLR, 2017.

\bibitem[Cohen \& Welling(2016{\natexlab{a}})Cohen and Welling]{cohen2016group}
Cohen, T. and Welling, M.
\newblock Group equivariant convolutional networks.
\newblock In \emph{International conference on machine learning}, pp.\  2990--2999. PMLR, 2016{\natexlab{a}}.

\bibitem[Cohen et~al.(2019{\natexlab{a}})Cohen, Weiler, Kicanaoglu, and Welling]{cohen2019gauge}
Cohen, T., Weiler, M., Kicanaoglu, B., and Welling, M.
\newblock Gauge equivariant convolutional networks and the icosahedral cnn.
\newblock In \emph{International conference on Machine learning}, pp.\  1321--1330. PMLR, 2019{\natexlab{a}}.

\bibitem[Cohen \& Welling(2016{\natexlab{b}})Cohen and Welling]{cohen2016steerable}
Cohen, T.~S. and Welling, M.
\newblock Steerable cnns.
\newblock \emph{arXiv preprint arXiv:1612.08498}, 2016{\natexlab{b}}.

\bibitem[Cohen et~al.(2018)Cohen, Geiger, K{\"o}hler, and Welling]{cohen2018spherical}
Cohen, T.~S., Geiger, M., K{\"o}hler, J., and Welling, M.
\newblock Spherical cnns.
\newblock \emph{arXiv preprint arXiv:1801.10130}, 2018.

\bibitem[Cohen et~al.(2019{\natexlab{b}})Cohen, Geiger, and Weiler]{cohen2019general}
Cohen, T.~S., Geiger, M., and Weiler, M.
\newblock A general theory of equivariant cnns on homogeneous spaces.
\newblock \emph{Advances in neural information processing systems}, 32, 2019{\natexlab{b}}.

\bibitem[De~Haan et~al.(2020)De~Haan, Weiler, Cohen, and Welling]{de2020gauge}
De~Haan, P., Weiler, M., Cohen, T., and Welling, M.
\newblock Gauge equivariant mesh cnns: Anisotropic convolutions on geometric graphs.
\newblock \emph{arXiv preprint arXiv:2003.05425}, 2020.

\bibitem[Finzi et~al.(2020)Finzi, Stanton, Izmailov, and Wilson]{finzi2020generalizing}
Finzi, M., Stanton, S., Izmailov, P., and Wilson, A.~G.
\newblock Generalizing convolutional neural networks for equivariance to lie groups on arbitrary continuous data.
\newblock In \emph{International Conference on Machine Learning}, pp.\  3165--3176. PMLR, 2020.

\bibitem[Gao et~al.(2022)Gao, Lin, and Zhu]{gao2022deformation}
Gao, L., Lin, G., and Zhu, W.
\newblock Deformation robust roto-scale-translation equivariant {CNN}s.
\newblock \emph{Transactions on Machine Learning Research}, 2022.
\newblock ISSN 2835-8856.
\newblock URL \url{https://openreview.net/forum?id=yVkpxs77cD}.

\bibitem[Gruver et~al.(2022)Gruver, Finzi, Goldblum, and Wilson]{gruver2022lie}
Gruver, N., Finzi, M., Goldblum, M., and Wilson, A.~G.
\newblock The lie derivative for measuring learned equivariance.
\newblock \emph{arXiv preprint arXiv:2210.02984}, 2022.

\bibitem[Kondor \& Trivedi(2018)Kondor and Trivedi]{kondor2018generalization}
Kondor, R. and Trivedi, S.
\newblock On the generalization of equivariance and convolution in neural networks to the action of compact groups.
\newblock In \emph{International Conference on Machine Learning}, pp.\  2747--2755. PMLR, 2018.

\bibitem[Krizhevsky et~al.(2017)Krizhevsky, Sutskever, and Hinton]{krizhevsky2017imagenet}
Krizhevsky, A., Sutskever, I., and Hinton, G.~E.
\newblock Imagenet classification with deep convolutional neural networks.
\newblock \emph{Communications of the ACM}, 60\penalty0 (6):\penalty0 84--90, 2017.

\bibitem[Lang \& Weiler(2021)Lang and Weiler]{lang2021a}
Lang, L. and Weiler, M.
\newblock A wigner-eckart theorem for group equivariant convolution kernels.
\newblock In \emph{International Conference on Learning Representations}, 2021.
\newblock URL \url{https://openreview.net/forum?id=ajOrOhQOsYx}.

\bibitem[LeCun et~al.(1989)LeCun, Boser, Denker, Henderson, Howard, Hubbard, and Jackel]{lecuncnn89}
LeCun, Y., Boser, B., Denker, J.~S., Henderson, D., Howard, R.~E., Hubbard, W., and Jackel, L.~D.
\newblock Backpropagation applied to handwritten zip code recognition.
\newblock \emph{Neural Comput.}, 1\penalty0 (4):\penalty0 541–551, dec 1989.
\newblock ISSN 0899-7667.
\newblock \doi{10.1162/neco.1989.1.4.541}.
\newblock URL \url{https://doi.org/10.1162/neco.1989.1.4.541}.

\bibitem[Li et~al.(2021)Li, Zhang, and Arora]{li2021why}
Li, Z., Zhang, Y., and Arora, S.
\newblock Why are convolutional nets more sample-efficient than fully-connected nets?
\newblock In \emph{International Conference on Learning Representations}, 2021.
\newblock URL \url{https://openreview.net/forum?id=uCY5MuAxcxU}.

\bibitem[Mairal(2016)]{mairal2016end}
Mairal, J.
\newblock End-to-end kernel learning with supervised convolutional kernel networks.
\newblock \emph{Advances in neural information processing systems}, 29, 2016.

\bibitem[Mairal et~al.(2014)Mairal, Koniusz, Harchaoui, and Schmid]{mairal2014convolutional}
Mairal, J., Koniusz, P., Harchaoui, Z., and Schmid, C.
\newblock Convolutional kernel networks.
\newblock \emph{Advances in neural information processing systems}, 27, 2014.

\bibitem[Mallat(2012)]{mallat2012group}
Mallat, S.
\newblock Group invariant scattering.
\newblock \emph{Communications on Pure and Applied Mathematics}, 65\penalty0 (10):\penalty0 1331--1398, 2012.

\bibitem[Qiu et~al.(2018)Qiu, Cheng, Sapiro, et~al.]{qiu2018dcfnet}
Qiu, Q., Cheng, X., Sapiro, G., et~al.
\newblock Dcfnet: Deep neural network with decomposed convolutional filters.
\newblock In \emph{International Conference on Machine Learning}, pp.\  4198--4207. PMLR, 2018.

\bibitem[Rahimi \& Recht(2007)Rahimi and Recht]{rahimi2007random}
Rahimi, A. and Recht, B.
\newblock Random features for large-scale kernel machines.
\newblock \emph{Advances in neural information processing systems}, 20, 2007.

\bibitem[Raj et~al.(2017)Raj, Kumar, Mroueh, Fletcher, and Sch{\"o}lkopf]{raj2017local}
Raj, A., Kumar, A., Mroueh, Y., Fletcher, T., and Sch{\"o}lkopf, B.
\newblock Local group invariant representations via orbit embeddings.
\newblock In \emph{Artificial Intelligence and Statistics}, pp.\  1225--1235. PMLR, 2017.

\bibitem[Reisert \& Burkhardt(2007)Reisert and Burkhardt]{JMLR:v8:reisert07a}
Reisert, M. and Burkhardt, H.
\newblock Learning equivariant functions with matrix valued kernels.
\newblock \emph{Journal of Machine Learning Research}, 8\penalty0 (15):\penalty0 385--408, 2007.
\newblock URL \url{http://jmlr.org/papers/v8/reisert07a.html}.

\bibitem[Schuchardt et~al.(2023)Schuchardt, Scholten, and G{\"u}nnemann]{schuchardt2023provable}
Schuchardt, J., Scholten, Y., and G{\"u}nnemann, S.
\newblock Provable adversarial robustness for group equivariant tasks: Graphs, point clouds, molecules, and more.
\newblock \emph{Advances in Neural Information Processing Systems}, 36:\penalty0 197--252, 2023.

\bibitem[Schölkopf \& Smola(2018)Schölkopf and Smola]{Smolakernel}
Schölkopf, B. and Smola, A.~J.
\newblock \emph{{Learning with Kernels: Support Vector Machines, Regularization, Optimization, and Beyond}}.
\newblock The MIT Press, 06 2018.
\newblock ISBN 9780262256933.
\newblock \doi{10.7551/mitpress/4175.001.0001}.
\newblock URL \url{https://doi.org/10.7551/mitpress/4175.001.0001}.

\bibitem[Shalev-Shwartz \& Ben-David(2014)Shalev-Shwartz and Ben-David]{shalev2014understanding}
Shalev-Shwartz, S. and Ben-David, S.
\newblock \emph{Understanding machine learning: From theory to algorithms}.
\newblock Cambridge university press, 2014.

\bibitem[Takesaki et~al.(2003)]{takesaki2003theory}
Takesaki, M. et~al.
\newblock \emph{Theory of operator algebras II}, volume 125.
\newblock Springer, 2003.

\bibitem[Weiler \& Cesa(2019)Weiler and Cesa]{weiler2019general}
Weiler, M. and Cesa, G.
\newblock General e (2)-equivariant steerable cnns.
\newblock \emph{Advances in Neural Information Processing Systems}, 32, 2019.

\bibitem[Weiler et~al.(2018{\natexlab{a}})Weiler, Geiger, Welling, Boomsma, and Cohen]{weiler20183d}
Weiler, M., Geiger, M., Welling, M., Boomsma, W., and Cohen, T.~S.
\newblock 3d steerable cnns: Learning rotationally equivariant features in volumetric data.
\newblock \emph{Advances in Neural Information Processing Systems}, 31, 2018{\natexlab{a}}.

\bibitem[Weiler et~al.(2018{\natexlab{b}})Weiler, Hamprecht, and Storath]{weiler2018learning}
Weiler, M., Hamprecht, F.~A., and Storath, M.
\newblock Learning steerable filters for rotation equivariant cnns.
\newblock In \emph{Proceedings of the IEEE Conference on Computer Vision and Pattern Recognition}, pp.\  849--858, 2018{\natexlab{b}}.

\bibitem[Zhang et~al.(2008)Zhang, Tsang, and Kwok]{zhang2008improved}
Zhang, K., Tsang, I.~W., and Kwok, J.~T.
\newblock Improved nystr{\"o}m low-rank approximation and error analysis.
\newblock In \emph{Proceedings of the 25th international conference on Machine learning}, pp.\  1232--1239, 2008.

\bibitem[Zhang et~al.(2017)Zhang, Liang, and Wainwright]{zhang2017convexified}
Zhang, Y., Liang, P., and Wainwright, M.~J.
\newblock Convexified convolutional neural networks.
\newblock In \emph{International Conference on Machine Learning}, pp.\  4044--4053. PMLR, 2017.

\end{thebibliography}
\bibliographystyle{icml2024}

\newpage
\appendix
\onecolumn
\section{Some Useful Mathematical Tools}

We state the classical result of characterizing a Reproducing Kernel Hilbert Space (RKHS) of functions defined from Hilbert space mappings.

\begin{theorem}
\label{RKHS-main-theorem}
    Let $\phi: \mathcal{X} \to H$ be a feature map to a Hilbert space $H$, and let $K(x,x') \coloneqq \langle \phi(x),\phi(x') \rangle_H$ for $x,x' \in \mathcal{X}$. Let $\mathcal{H}$ be the linear subspace defined by $\mathcal{H} \coloneqq \{ f_w, w \in H \}$ such that $f_w : x \mapsto \langle w,\phi(x) \rangle_H$, and we consider the norm $||f_w||^2_\mathcal{H} \coloneqq inf_{w' \in H} \{ ||w'||^2_H$ such that $f_w = f_{w'} \}$. Then $\mathcal{H}$ is the RKHS associated to the kernel $K$.
\end{theorem}

We now state another classical result, from harmonic analysis that is used to prove stability results of equiv-CKNs.

\begin{lemma}[Schur's test]
   Let $\mathcal{H}$ be a Hilbert space and $\Omega$ a subset of $\mathbb{R}^d$. Consider $T$ an integral operator with kernel\footnote{This type of kernel is known as Schwartz kernel.} $k: \Omega \times \Omega \to \mathbb{R}$ such that for all $u \in \Omega$ and $x \in L^2(\Omega,\mathcal{H})$,
   \begin{equation}
       Tx(u) = \int_\Omega k(u,v)x(v)dv.
   \end{equation}
   If $\int |K(u,v)|dv \leq C$ and $\int |K(u,v)|du \leq C$ for all $u \in \Omega$ and $v \in \Omega$ respectively, for some constant $C$, then for all $x \in L^2(\Omega,\mathcal{H})$, we have $Tx \in L^2(\Omega,\mathcal{H})$ and $||T|| \leq C$. 
\end{lemma}

For an operator $T: L^2(\mathbb{R}^d,\mathcal{H}) \to L^2(\mathbb{R}^d,\mathcal{H'})$, the norm is defined as $||T|| \coloneqq sup_{||x||_{L^2(\mathbb{R}^d,\mathcal{H})} \leq 1} ||Tx||_{L^2(\mathbb{R}^d,\mathcal{H}')}$. One can extend this definition of operator norm on $L^2(G)$, as the latter is the base of our signals defined on the group $G$, rather than on $\mathbb{R}^d$. With the support of Haar measure on locally compact group $G$ which supports the signal domain the structure of norm is similar to that of on $L^2(\mathbb{R}^d)$ (with a Lebesgue measure support).

\section{Further Details on Group Equivariant CKNs on Euclidean Domain}
\label{Grp-equiv-ckns-constructions}

\textbf{Patch extraction operator} $P_k$'s, given by \cref{patch-equation} which is encoded in a Hilbert space $\mathcal{P}_k$, preserves the norm, i.e., $||P_kx_{k-1}|| = || x_{k-1}||$, because of $\mathcal{P}_k$'s are supported by normalized Haar measure. Hence $P_kx_{k-1} \in L^2(G,\mathcal{P}_k)$.

\textbf{Kernel mapping operator} $M_k$'s. Here we give a detailed description of the operator defined in \cref{kernel-equation} and the choice of dot-product kernels. As defining a homogeneous dot-product kernel yields $M_k$'s as point-wise operator and hence commutes well with the group action $L_g$'s for $g \in G$, we stick to the definition of kernel mapping operators given by \cite{bietti2019group}.

\begin{lemma}[Lemma 1, \cite{bietti2019group}]
   Let $K_k$ be a positive-definite kernel given by \cref{positive-definite kernel} which satisfies the constraints given by $k_k$'s. Then the RKHS mapping $\varphi_k : \mathcal{P}_k \to \mathcal{H}_k$, for all $x,x' \in \mathcal{P}_k$ satisfies $||\varphi_k(x) - \varphi_k(x')|| \leq ||x-x'||$. Moreover $K_k(x,x') \geq \langle x,x' \rangle$, i.e., the kernel $K_k$'s are lower bounded by the linear kernels.
\end{lemma}
\begin{proof}
For the proof we make use of the fact from the Maclaurin expansion\footnote{We also assume that the series $\sum_ib_i$ and $\sum_i ib_i$'s are convergent.} of $k_k$'s that
\begin{equation}
    k_k(u) = k_k(1) - \int_u^1 k_k'(t)dt \geq k_k(1) - k_k'(1)(1-u),
\end{equation}
for all $u \in [-1,+1]$. Then for $x,x' \neq 0$ we have
\begin{equation*}
    ||\varphi_k(x) - \varphi_k(x')||^2 = ||x||^2 + ||x'||^2 - 2||x||||x'||k_k(u),
\end{equation*}
with $u = \langle x,x' \rangle / (||x||||x'||)$. Using the above inequality and the constraint $k_k(1)=1$ we have
\begin{equation*}
\begin{split}
    ||\varphi_k(x) - \varphi_k(x')||^2 & \leq ||x||^2 + ||x'||^2 - 2||x|||x'||(1-k_k'(1)+k_k'(1)u) \\
                                    & = (1-k_k'(1))(||x||^2+||x'||^2-2||x|||x'||) \\
                                    & + k_k'(1)(||x||^2+||x'||^2 - 2\langle x,x' \rangle) \\
                                    & = (1-k_k'(1))|||x|| - ||x'|||^2 + k_k'(1)||x-x'||^2 \\
                                    & \leq ||x-x'||^2.
\end{split}
\end{equation*}
For the last inequality we use the fact that $0 \leq k_k'(1) \leq 1$.
\end{proof}
\begin{remark}
One can extend the above lemma for any Lipschitz continuous mapping with $\varphi_k(\cdot)$ being $\rho$-Lipschitz with $\rho = max(1,\sqrt{k_k'(1)})$, for any value of $k_k'(1)$. Then similarly the above inequality will hold and more generally we'll also have $||\varphi_k(x) - \varphi_k(x')||^2 \leq k_k'(1)||x-x'||^2$ when $k_k'(1) \geq 1$. This together with the above inequality gives us $||\varphi_k(x) - \varphi_k(x')||^2 \leq \rho^2||x-x'||^2$, and yields the result. However for the sake of simplicity we just avoid using Lipschitz continuous kernel mapping as otherwise the stability constants would also depend upon $\rho$ which would increase exponentially with the number of layers that one wants to avoid.

For example, homogeneous Gaussian kernel defined as, $K_{RBF}(x,x') = exp(-\alpha||x-x'||^2)$ is non-expansive only when $\alpha \leq 1$ but is still Lipschitz for any values of $\alpha$.
\end{remark}

\textbf{Pooling operator} $A_k$'s. In the definition of $A_k$ in \cref{pooling-int} the pooling filter $h_k$ is typically localized around the identity element of $G$. By applying Schur's test on the operator $A_k$ one obtains that $||A_k|| \leq 1$ and hence $x_k(u) \in L^2(G,\mathcal{H}_k)$.

\begin{remark}
Unlike the operators $P_k$ and $M_k$, $A_k$ doesn't preserve the norm (which is in contrary to the setting of \cite{mallat2012group}) as $||A_kx_k(u)|| \leq ||x_k(u)||$. As we are using a pooling filter with a scale of $\sigma_k$, therefore $A_k$'s may reduce frequencies of signals that are larger than $1/\sigma_k$. However norm preservation is less relevant in the kernel based setting as discussed in \cite{bietti2019group}, as if one picks a Gaussian kernel mapping on top of the last feature map instead of a linear layer as prediction layer then the final feature representation preserves stability as well as have a unit norm.
\end{remark}

\begin{remark}
    One can also pool on subset $H \subseteq G$ by only integrating on $H$, much like the subgroup pooling described in \cite{cohen2016group} for group equiv-CNNs. This subsampling on a subgroup $H \subseteq G$, though gives the subsampled feature map $H$-equivariant but one can obtain the full group $G$-equivariance by performing the pooling on the entire $H$. Moreover from the first expression of $A_k$ in \cref{pooling-int} it is easy to see that the pooling operator commutes with $L_g$.
\end{remark}

\textbf{Some notes on discretization and kernel approximation.} Though for our theoretical analysis purposes we have defined signals on $L^2(G,\mathcal{H}_k)$ but for practical implementation one needs to discretize the signals as in practice, signals are discrete. For group equiv-CNNs it is nicely discussed in \cite{cohen2016steerable,cohen2019general} through the notion of fiber space (bundles), making each discrete feature maps equivariant and hence the entire network equivariant, through the efficient implementation of G-equivariant layers. For our construction it is possible to sample each feature map $\Phi_k(x) \coloneqq x_k(u)$ on a discrete set with no loss of information. For the classical CKNs an in-depth discussion on discretization is available through section 2.1 of \cite{bietti2019group} or by simply following the construction of hierarchical CKN layers from \cite{mairal2016end}.

In \cite{mairal2016end} a finite dimensional subspace projection of RKHS mappings $\varphi_k(\cdot)$ are discussed through an adapted Nystr\"{o}m method \cite{zhang2008improved} which is essential in the construction of CKNs. However this is not a drawback as such finite dimensional approximation of RKHS mappings still live in the corresponding RKHSs as well as it won't hurt the stability results due to the non-expansiveness of the projection. However in this case some signal information is lost as through projection we can no longer maintain the norm preservence of the kernel mapping operator $M_k$.

\textbf{Equivariant convolutional kernel representations} 

\begin{corollary}[Equivariant kernels]
    \cref{pooling-int} can always be written as cross-correlation between the feature map and the pooling filter. Moreover in equiv-CKNs, representation, $\Phi_N(x) \in L^2(G,\mathcal{H}_N)$ is equivariant (with respect to $G$) if and only if each $\varphi_k$'s are in cross-correlation with an equivariant pooling filter.
\end{corollary}
\vspace{-2mm}
\begin{proof}
    The proof is straight-forward and immediately follows from the definition of cross-correlation, i.e., $[h_k \ast x_k](u) \coloneqq \int_G h_k(u^{-1}v)x_k(v)d\mu_k(v) = A_kx_k(u)$. For the second part, note that $A_kx_k(u)$ can be written as $A_kM_kP_kx_{k-1}(u)$ as one can see it from \cref{ckn-schematic}. Then establishing link with kernel mapping $\varphi_k$ with $h_k$'s are staightforward and the equivariance followed from \cref{thm: equiv-ckns}.
\end{proof}

\begin{remark}
    Note that through the above corollary we get another equivalent notion of equivariant kernels, as described in (section 3.1 of \cite{cohen2019general}). However note that in equiv-CKNs the kernels are described by kernel mapping $\varphi_k$'s which is given by the RKHS mapping, giving true flavour of kernel machine, which is missing in group equiv-CNNs. We note that more recently \cite{lang2021a} gives a full characterization of group equivariant kernels but it still misses the notion of RKHSs.
\end{remark}

\section{Stability Analysis of Equivariant Convolutional Kernel Representations}
\label{Stability-proofs}

Before giving the proofs of \cref{lemma:Stability-of-ckns} and \cref{thm:Stability-bound} we first dive deep into the stability form and how it is controlled by the operator norm (and hence of the RKHSs norm) which are motivated by similar notion of diffeomorphism studied in \cite{mallat2012group}. 

The assumption $sup_{c \in \hat{S}_k} |c| \leq \kappa\sigma_{k-1}$ is made to relate the scale of pooling operator at layer $k-1$ with the diameter of the patch $S_k$. As $\sigma_k$'s increases exponentially with the layers $k$ and characterizes resolution of each feature map, the assumption helps us to consider such patch sizes that are adapted to those resolutions, and helps us control the stability. Let us first state the bound on operator norms.

\begin{proposition}[Proposition 4 \cite{bietti2019group}]
\label{bounds-operator-norm}
    For any $x \in L^2(\mathbb{R}^d,\mathcal{H}_0)$, we have
    \begin{equation}
    \begin{split}
        ||\Phi_N(L_\tau x) - \Phi_N(x)|| & \leq ( \sum_{k=1}^N ||[P_kA_{k-1},L_\tau]|| \\ 
        & + ||[A_N,L_\tau]|| \\
        & + ||L_\tau A_N - A_N||)\cdot ||x||. 
    \end{split}
    \end{equation}
\end{proposition}

By expanding $\Phi_N$'s as shown in the multilayered construction of CKNs in \cref{Stability equiv-CKNs} and using the facts of norm preservence of $P_k$ and $M_k$'s, non-expansiveness of $M_k$'s and $||A_k|| \leq 1$ we can get the above result. Moreover one also uses the fact that kernel mapping $M_k$ is defined point-wise and thus commutes with the deformation operator $L_\tau$. The result holds even when $x$ is defined on the locally compact group $G$, i.e., when $x \in L^2(G,\mathcal{H}_0)$.

\section{Geometric Model Complexity of Deep Equivariant Convolutional Representations}
\label{geo-complexity-equiv-cnns}

If one can write a group equiv-CNN $f$ in the form $f(x) = \langle f, \Phi(x) \rangle$, where $\Phi(\cdot)$ is the equivariant convolutional kernel representation, then one can extend the stability analysis of equiv-CKNs, $\Phi(\cdot)$'s to the stability analysis of equiv-CNNs. Moreover computing the RKHS norm of the equiv-CNNs one can also control generalization, so that controlling the RKHS norm serves as the geometric model complexity of equiv-CNNs, where the term `geometric' refers to the equivariance of operators and the geometry of RKHSs.

Before outlining the construction of an equiv-CNNs in RKHSs, let's state a lemma from \cite{bietti2019group} which closely follows the results of \cite{zhang2017convexified}, linking the homogeneous activation function with RKHSs $\mathcal{H}_k$, which we believe also holds for group equiv-CNNs as the pointwise homogeneous activation maps $\sigma$ are replaced with pointwise non-linearity maps $\nu$, as described in \cite{cohen2016group}.

\begin{lemma}[Lemma 11, \cite{bietti2019group}]
    If the activation maps $\sigma$ admits a polynomial expansion and we define our kernel $K_k$ as given in \cref{positive-definite kernel}. Then for $g \in \mathcal{P}_k$, the RKHS $\mathcal{H}_k$ contains the function,
    \begin{equation}
    \label{intermediate}
        f: x \mapsto ||x||\sigma(\langle g,x \rangle / ||x||),
    \end{equation}
    which matches the form given by \cref{equiv-cnns-rkhs}.
\end{lemma}

For our construction of $k_k$'s, the next corollary follows from the above lemma as well as from the \cref{RKHS-main-theorem}.

\begin{corollary}
\label{linear-function}
    The RKHSs $\mathcal{H}_k$ contain all linear functions of the form $x \mapsto \langle g,x \rangle$, with $g \in \mathcal{P}_k$.
\end{corollary}

Note that RKHS of the kernel $K_N(x,x') = \langle \Phi(x), \Phi(x') \rangle$, defined at the prediction layer as final representation $\Phi(x) \in \mathcal{H}_{N+1}$ contains functions of the form $f: x \mapsto \langle w, \Phi(x) \rangle$, with $w \in \mathcal{H}_{N+1}$ and $||f|| \leq ||w||_{\mathcal{H}_{N+1}}$. This is a consequence of \cref{RKHS-main-theorem}, and also in line with the stated corollary, as in our construction $\mathcal{P}_k$'s are also RKHS.

\subsection{Note on the norm of equiv-CNN $f$ and generalization bounds}
\label{generalization-Rademacher}

We have seen that how the operator norms control the stability of the CKNs and through \cref{lipschitz-smoothness} we get the model complexity of group equiv-CNNs, where the RKHS norm of $f$ also plays an important role in the stability of the model as well as understanding the generalization capabilities, and hence of the geometric model complexity of the equivariant convolutional networks.

One can study generalization bounds through Rademacher complexity and margin bounds, for e.g., as done in \cite{shalev2014understanding}, where one studies the upper bound on the Rademacher complexity of a function class $\mathcal{F}_\lambda$ with bounded RKHS norm, $\mathcal{F}_\lambda = \{f \in \mathcal{H}_K : ||f|| \leq \lambda \}$, for a dataset $\{x_1,x_2,...,x_M\}$, given by, 
\begin{equation*}
    Rad_M(\mathcal{F}_\lambda) \leq \frac{\lambda \sqrt{1/M \sum_{i=1}^M K(x_i,x_i)}}{\sqrt{M}}.
\end{equation*}
The bound remains valid when considering CNN functions of form $f_\sigma$, given by \cref{intermediate}, as such family of functions $f_\sigma$ contains in the class of $\mathcal{F}_\lambda$. Generalization bound depends upon the model complexity parameter $\lambda$, sample size $M$ and on the choice of the kernel at the prediction layer. However it doesn't explicitly yield the layer-wise architectural choices of CKNs. However in practice, learning with a tight constraint, like $||f|| \leq \lambda$, can be infeasible and thus one needs to replace $\lambda$ with a similar bound with $||f_M||$ which can be directly obtained from the training data (Theorem 26.14,\cite{shalev2014understanding}). This then involves the construction of equiv-CNNs in a RKHS, as seen in \cref{construction-formalities}. and the corresponding RKHS norm, together with the sample size gives the upper bound of Rademacher complexity. Hence this leads to a way of studying generalization bounds of group equiv-CNNs.


\end{document}